\documentclass[11pt,a4paper]{amsart}

\usepackage{ifxetex}
\usepackage{ifpdf}
\ifxetex
  \usepackage[T1]{fontenc}
\else
  \usepackage[utf8]{inputenc}
\fi
\usepackage{amssymb, amsthm, amsmath}
\ifpdf
  \usepackage{graphicx}
  \usepackage{hyperref}
\else 
  \ifxetex 
    \usepackage{graphicx}
    \usepackage{hyperref}
  \else 
    \usepackage[dvipdfm]{graphicx}
    \usepackage[dvipdfm]{hyperref}
  \fi
\fi

\usepackage{soul}
\usepackage{color}
\usepackage{enumerate}
\usepackage{subfig}
\usepackage[margin=1in]{geometry} 
\usepackage[ruled]{algorithm2e}

\theoremstyle{definition}
\newtheorem{none}{Part}[section] 
\newtheorem{definition}{Definition}[section]
\newtheorem{remark}[definition]{Remark}

\newtheorem{proposition}[definition]{Proposition}

\newcommand\R{\mathbb{R}}     
\ifxetex
  \renewcommand\C{\mathbb{C}} 
\else
  \newcommand\C{\mathbb{C}}
\fi

\newcommand\D{\mathbb{D}}
\newcommand\PS{\mathbb{P}}

\newcommand\SE{\operatorname{SE}(3)}
\newcommand\mbf{\mathbf}
\renewcommand\H{\mathbb{H}}

\begin{document}

\title{Inverse Kinematics of Some General 6R/P Manipulators}

\author[J.~Capco]{Jose Capco}
\address[J.~Capco]{Research Institute for Symbolic Computation, Johannes Kepler University, Austria}
\email{jcapco@risc.jku.at}

\author[M.J.C.~Loquias]{Manuel Joseph C.~Loquias}
\address[M.J.C.~Loquias]{Institute of Mathematics, College of Science, University of the Philippines Diliman, Quezon City 1101, Philippines}
\email{mjcloquias@math.upd.edu.ph}

\author[S.M.M.~Manongsong]{Saraleen Mae M.~Manongsong}
\address[S.M.M.~Manongsong]{Institute of Mathematics, College of Science, University of the Philippines Diliman, Quezon City 1101, Philippines}
\email{smmanongsong@gmail.com}

\author[F.R.~Nemenzo]{Fidel R.~Nemenzo}
\address[F.R.~Nemenzo]{Institute of Mathematics, College of Science, University of the Philippines Diliman, Quezon City 1101, Philippines}
\email{fidel@math.upd.edu.ph}

\begin{abstract}
	We develop an algorithm that solves the inverse kinematics of general serial 2RP3R, 2R2P2R, 
  3RP2R and 6R manipulators based 
  on the HuPf algorithm. We identify the workspaces of the 3-subchains of the manipulator with a quasi-projective 
  variety in $\PS^7$ via dual quaternions. This allows us to compute linear forms that describe linear spaces containing the 
  workspaces of these 3-subchains. We present numerical examples that illustrate the algorithm and show the real 
  solutions.
\end{abstract}



\date{\today}

\maketitle

  \section{Introduction}
    The inverse kinematic (IK) problem in computational kinematics involves finding joint values of a manipulator for a specified position and orientation of its end-effector tool. The development of computer algebra systems led to notable improvements of obtaining solutions to this problem.  Chapelle and Bidaud \cite{chap} obtained a mathematical function that approximates joint values of general 6R manipulators through genetic programming. The same manipulator was studied by Wang et al.~\cite{wang} using Gröbner bases. The works of Husty et al.~\cite{ik6r} and Pfurner \cite{pfurner} gave an algebraic-geometric insight to the problem and addressed it using classical results in multi-dimensional geometry. Gan et al.~\cite{gan} solved the problem for the case of a 7-link 7R mechanism using dual quaternions and Dixon's resultant. 
    
    Most of the developments mentioned above deal with manipulators having purely revolute joints, while manipulators with prismatic joints have not yet been fully explored. Joints of manipulators in the industry are mostly prismatic or revolute. We provide a solution to the IK problem of some serial manipulators that contain prismatic joints, in particular, 2RP3R, 2R2P2R, 
    3RP2R and 6R manipulators, using an approach similar to that of Husty and Pfurner \cite{ik6r,pfurner} but based on the algebra of dual quaternions. The solution turns out to be not as easy and straightforward as the case of 6R manipulators. 
    
    The rest of the paper is organized as follows. We start with some basic concepts in Section~2. Hyperplanes for the 2-chains are computed in Section~3, while linear spaces for the workspaces are computed in Section~4. The procedure for solving the IK problem is shown in Section~5 along with some examples, and conclusions are discussed in Section~6.
      
    \section{Preliminaries}
    
    \subsection{Manipulator structure} A serial 6-chain manipulator with 2RP3R structure is a sequence of seven links connected by six joints starting from the base: two revolute joints (2R), a prismatic joint (P) and three more revolute joints (3R). See Figure~\ref{fig:2rp3rwig}. Serial 6-chain manipulators with 2R2P2R, 3RP2R and 6R structures are defined analogously.
    
    \begin{figure}
      \centering
      \includegraphics[scale=0.4]{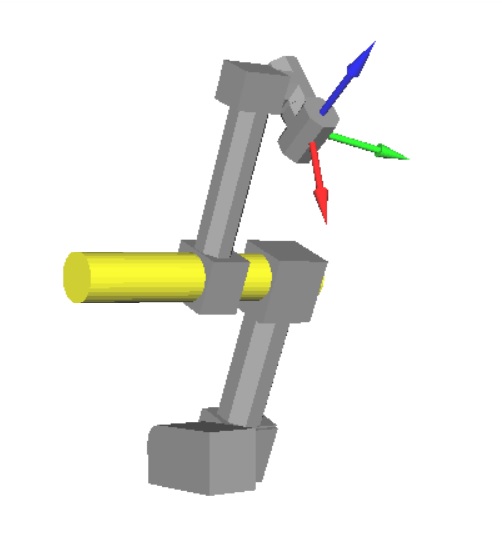}
      \caption{A 2RP3R Manipulator}
      \label{fig:2rp3rwig}
    \end{figure}
    
    For the analysis of this mechanism, we adopt the Denavit-Hartenberg (DH) convention on assigning coordinate frames \cite{spong}.  The base frame $F_1$ is associated to the base link in which the $z_1$-axis coincides with the first rotation axis, and the $x_1$- and $y_1$-axes are placed to form a right-handed coordinate frame. For $i=2,3,\ldots,6$, the $z_i$-axis is placed along the $i$th joint axis and the origin is at the intersection of the $z_i$-axis and the common normal between the $z_{i-1}$- and $z_i$-axes. The $x_i$-axis is set along the direction of this  normal and the $y_i$-axis is set to form a right-handed frame $F_i$. The end-effector frame $F_7$ (or EE frame) has its $z_7$-axis placed along the $z_6$-axis, and the $x_7$- and $y_7$-axes are placed to form a right-handed frame.

    The displacement between two consecutive frames $F_{i}$ and $F_{i+1}$ can be described using four parameters: \emph{rotation angle} $\theta_{i}$, \emph{offset} $d_{i}$, \emph{distance} $a_{i}$ and \emph{twist angle} $\alpha_{i}$, where $i=1,2,\ldots,6$. However, rotation and twist angles can be parametrized by points in the real projective space $\PS^1(\R)$. For simplicity we 
    will disregard half-rotations (or rotations that are odd multiples of $\pi$) which allows us to parametrize the angles by 
    half-angle tangents. That is, we use the parameters 
    $l_i := \tan\tfrac{\alpha_i}{2}$ and $v_i:= \tan\tfrac{\theta_i}{2} $
    which gives an algebraic description of the manipulator's workspace.

    We refer to $v_1,v_2,d_3,v_4,v_5$ and $v_6$ as the \textit{joint variables} and the rest of the parameters as the \textit{DH-parameters} of a 2RP3R manipulator. These variables and parameters are  referred analogously for 2R2P2R, 3RP2R and 6R manipulators. Note that $a_6=l_6=0$. We assume $d_1=d_6=0$.

    Rigid transformations in $\SE$ are usually represented by homogeneous transformation matrices. In particular, the displacement between frames $F_{i}$ and $F_{i+1}$ consists of a rotation by $\theta_i$, a translation $(a_i,0,d_i)$ and a rotation by $\alpha_i$, and is represented by the matrix product
    $$B_i=\left(
    \begin{array}{cccc}
      \cos(\theta_i) & -\sin(\theta_i) & 0 & 0 \\
      \sin(\theta_i) & \cos(\theta_i) & 0 & 0 \\
      0 & 0 & 1 & 0 \\
      0 & 0 & 0 & 1 \\
    \end{array}
    \right) \left(
    \begin{array}{cccc}
    1 & 0 & 0 & a_i \\
    0 & \cos(\alpha_i) & -\sin(\alpha_i) & 0 \\
    0 & \sin(\alpha_i) & \cos(\alpha_i) &  d_i \\
    0 & 0 & 0 & 1 \\
    \end{array}
    \right)$$
    for $i=1,2,\ldots,6.$ The kinematic equation of a 6-chain manipulator can be formulated as
    \begin{equation}
      B_1 B_2 \cdots B_6 = E \label{eq:ke}
    \end{equation}
    where $E$ is the matrix representing the displacement of the EE frame relative to the base frame. 
    
    Given the DH-parameters of a manipulator and a specified pose of its end-effector $E$, the inverse kinematics problem aims to find values of its joint variables for which equation~\eqref{eq:ke} holds. In general, it is difficult to solve this problem with trigonometric functions because the system of equations are nonlinear. To overcome this hurdle, the equations may be expressed algebraically via the kinematic map and tangent of half-angle substitutions (see the works of Husty~\cite{ik6r} and Pfurner~\cite{pfurner}). This allows one to express each rigid transformation $B_i$  above as points in the projective space $\PS^7(\R)$. Moreover, products of rigid transformations are obtained by identifying them as dual quaternions. Indeed, the algebra of dual quaternions is enough to directly obtain algebraic equations for solving the IK problem.
    
    \subsection{Dual quaternions}
    We write \textit{quaternions} as $ p = p_0+\mbf{p} $ where $p_0 \in \R$ and $\mathbf{p} \in \R^3$. The set of quaternions $\H$ is an algebra over $\R$ where addition and multiplication are given by
    \begin{equation*}\begin{split}
        p_0+\mbf{p} + q_0+\mbf{q}  &= p_0 +q_0 + \mathbf{p+q}, \textnormal{ and }\\
        (p_0+\mbf{p}) (q_0+\mbf{q}) &=  p_0 q_0 - \mathbf{p\cdot q}+ (p_0 \mathbf{q} + q_0 \mathbf{p} + \mathbf{p} \times \mathbf{q})
    \end{split} \end{equation*}
    respectively, for any $p_0+\mbf{p},q_0+\mbf{q}\in\H$, and where $\cdot$ is the usual inner product in $\R^3$. Moreover, $\lambda (p_0+\mbf{p}) =(p_0+\mbf{p})\lambda= \lambda p_0+\lambda\mbf{p}$ for any $\lambda \in \R$. The real numbers can be embedded in $\H$ via the map $r\mapsto r+\mbf{0}$. 
    
    The $\R$-algebra of \emph{dual quaternions}, denoted $\D$, consists pairs of quaternions that are written as a formal sum $\sigma = p + \epsilon q$ where  $p\in \H$  and $q \in \H$. Operations in $\D$ are as follows:
    \begin{align*}
    (p + \epsilon q) + (s + \epsilon t) &= (p+s) + \epsilon (q+t)\\
    (p + \epsilon q) (s + \epsilon t) &= ps + \epsilon (pt+qs)\\
    \lambda(p + \epsilon q) &= \lambda p + \epsilon \lambda q
    \end{align*}
    for any $p + \epsilon q$, $s + \epsilon t \in \D$ and $\lambda\in \R$. Here, we take $\epsilon^2=0$ and $\epsilon\neq0$. Note that $\H$ can be embedded in $\D$ via the map $p\mapsto p+\epsilon 0$.  
       
    The \emph{conjugate} of a quaternion $p=p_0+\mathbf{p}$ is $p^*:=p_0-\mathbf{p}$, and the conjugate of a dual quaternion $\sigma=p + \epsilon q=(p_0+\mathbf{p})+\epsilon(q_0+\mathbf{q})$ is $\sigma^* := p^* + \epsilon q^*$. It can be verified that $\sigma\sigma^*=\sigma^*\sigma$ is a nonnegative real number if and only if $p_0q_0+\mbf p \cdot \mbf q =0$. In this case, set $|\sigma|:=\sigma\sigma^*$.
    
    We consider the set
    $$\D_s:=\left\{\sigma\in\D \mid \sigma=p + \epsilon q \textnormal{ with } p\neq0 \textnormal{ and } \sigma\sigma^*\in \R_{\geq0}\right\}$$ 
    which is a subgroup of the group of units of $\D$. The 
    multiplicative inverse of $\sigma\in \D_s$ is $\sigma^{-1}=\tfrac{\sigma^*}{|\sigma|}$.
    
    Quaternions can be used to describe rotations in $\R^3$. For instance, a rotation by $\theta$ about a unit axis $\mbf{n}\in\R^3$ through the origin can be represented by a quaternion $\cos\tfrac{\theta}{2}+\sin\tfrac{\theta}{2}\mbf{n}$. On the other hand, dual quaternions can be used to describe rigid transformations in $\R^3$ \cite{selig}. For instance, the rigid transformation  determined by a rotation $p\in\H$ and a translation $t=0+\mbf{t}\in\H$ can be represented by the dual quaternion
    $$
    	\sigma=p+\epsilon \tfrac{1}{2}tp=(\cos\tfrac{\theta}{2}+\sin\tfrac{\theta}{2}\mbf{n})+\epsilon (-\tfrac{1}{2}\sin \tfrac{\theta}{2} \mbf{t}\cdot\mbf{n}+ \tfrac{1}{2}\cos \tfrac{\theta}{2}\mbf{t}+\tfrac{1}{2}\sin \tfrac{\theta}{2} (\mbf t\times\mbf n)).\label{eq:dqSE}
    $$
    \noindent 
	Naive storage and multiplication of homogenous transformation matrices is less efficient than 
	dual quaternions. Naive matrix multiplication involves 122 elementary operations, while naive 
	multiplication of dual quaternions requires 88 elementary operations. There are other 
	advantages of dual quaternions over matrices such as in interpolations and handling numerical 
	errors (see \cite{markley}).

The kinematic equation in \eqref{eq:ke} may now then be expressed as
    \begin{equation*}
    	\sigma_1\sigma_2\cdots \sigma_6=\sigma_E \label{eq:kineq}
    \end{equation*}
where $\sigma_E$ is a dual quaternion representing the displacement of the EE frame $F_7$ relative to the base frame $F_1$. The DH-parameters and joint variables of the manipulator are encoded in each $\sigma_i$. To solve the IK problem, we decompose the 6-chain into two parts -- the \textit{left} and \textit{right chains} -- consisting of transformations that correspond to the left and right sides of the equation    
     \begin{equation}
     \sigma_1\sigma_2\sigma_3=\sigma_E\sigma_6^{-1}\sigma_5^{-1}\sigma_4^{-1} \label{eq:kineqdec}
    \end{equation} 
    respectively. Hence, the problem is reduced to solving for joint variables in equation~\eqref{eq:kineqdec} for which frame $F_4$ on the left and on the right chains coincide. 

    \subsection{Workspaces} 

   We identify elements in $\SE$ with points in the real projective space $\PS^7(\R)$ via dual quaternions of the form $\sigma=p+\epsilon\tfrac{1}{2}tp$. Nonzero scalar multiples of $\sigma$ represent the same elements in $\SE$. 
    
    Writing $\sigma=(x_0+\mbf{x})+\epsilon(y_0+\mbf{y})$ where $\mbf{x}=(x_1,x_2,x_3)$ and $\mbf{y}=(y_1,y_2,y_3)$ as an 8-tuple, it can be verified that $\sigma$ is in the set   
    $$\{(x_0:x_1:\cdots:y_3)\in\mathbb{P}^7(\R) \,\big{|}\, \sum_{i=0}^3 x_iy_i =0 \text{ and } \sum_{i=0}^3 x_i^2 \ne 0\}.$$
   We call $(x_0:x_1:\dots:y_3)$ as the \textit{Study parameters} of the rigid transformation. Thus we may also identify $\SE$ with the above subset of $\mathbb{P}^7(\R)$. In this identification, $\SE$ is the intersection of the Study quadric $S$
    \begin{equation}
    x_0y_0 + x_1y_1+x_2y_2 +x_3y_3=0 \label{eq:study}
    \end{equation}
     with the complement of the linear space 
    $x_0=x_1=x_2=x_3=0$ in $\mathbb{P}^7(\R)$.  Recall that a linear space is the intersection of a number of hyperplanes. 
    
    Consider the frame $F_4$ of a 2RP3R chain. Given the DH-parameters and a pose of the end-effector, the set of all possible poses of 
    $F_4$ relative to $F_1$ on the left 2RP chain is called the \textit{left workspace}, 
    while all possible poses for $F_4$ relative to $F_1$ on the right 3R chain is called the 
    \textit{right workspace}. These two workspaces in  $\SE$  are obviously the same and 
    described by the given DH-parameters.
    
    Just as in \cite{pfurner}, we build upon Selig's theory \cite{selig} that the workspace of a 2-chain lies in a linear 3-space to describe the left and right workspaces as intersections of hyperplanes in $\PS^7(\R)$ and $\SE$. The hyperplane equations then serve as constraints of these workspaces. To solve the IK problem we regard the workspace of a 3-chain as the workspace of a 2-chain parametrized by a joint. For instance in a 2RP3R manipulator, the 
    left workspace could be the kinematic image of a 2R-chain parametrized by $d_3$ while the right workspace is the 
    kinematic image of a 2R-chain parametrized by say $v_6$. 
    
    Some advantages of looking at the workspace in the projective setting are (1) it allows us to geometrically analyze dimensions of intersection of hyperplanes before solving the IK problem,  and (2) we solve IK only using basic techniques in linear algebra and algebraic geometry \cite{ik6r}. 
    
    \section{Computing hyperplanes for the left chain}
    
    Let $\sigma_i=R_{z}(v_i)T_{z}(d_i) T_{x}(a_i) R_{x}(l_i)$ in $\SE$, where $R_z$ and $R_x$ are rotations about the $z$- and  $x$-axes, and $T_z$ and $T_x$ are translations along $z$- and $x$-axes, respectively. Note that $\sigma_i=\sigma_i(v_i)$ (a function of $v_i$) if the $i$th joint is revolute, and $\sigma_i=\sigma_i(d_i)$ if the $i$th joint is prismatic. Along the same axis, note also that $R_{z}$ and $T_{z}$ commute; likewise, $R_{x}$ and $T_{x}$ commute. 
    
    We want to get the  linear forms parametrized by joint $v_1$ which define the linear spaces that contain the kinematic image of the left chain
    $$V_L:=\{\sigma_1(v_1)\sigma_2(v_2) \sigma_3(d_3) \mid v_1,v_2,d_3 \in \R \}.$$ To this end, we first compute the linear space that contains the kinematic image (leaving out the first joint)
    $$V_1:=\{R_{z}(v_2) T_{x}(a_2) R_{x}(l_2)  T_{z}(d_3) \mid v_1,d_3\in \R \}$$
    where the fixed transformations $\sigma_1 T_{z}(d_2)$ and $R_{z}(v_3) T_{x}(a_3) R_{x}(l_3)$ are removed (recall that we assume $d_1=0$ and, except for the joint variables $v_1,v_2,d_3$, the other DH-parameters are fixed). Our aim is to find hyperplanes     
    $$ax_0+bx_1+cx_2+dx_3+ey_0+fy_1+gy_2+hy_3=0$$    
    in $\PS^7(\R)$ whose intersection contain $V_1$. To do this, we substitute the Study parameters of $V_1$  into this equation thereby obtaining a polynomial  equation in $v_2$ and $d_3$
    $$z_0v_2d_3+z_1v_2+z_2d_3+z_3=0$$
    with coefficients
    \begin{center}
      \begin{tabular}{l}
        $z_0=-4 e+4l_2f$\\        
        $z_1=8l_2c+8 d+4 a_2 g-4 a_2l_2 h$\\
        $z_2=-4l_2g+4h$\\
        $z_3=8 a+8 l_2b-4 a_2l_2e+4 a_2 f.$
      \end{tabular}
    \end{center}
    
    \noindent Since $z_i=0$ for $i\in\left\{0,1,2,3\right\}$, the system of equations above can be written in the following matrix form:

    $$\left(
    \begin{array}{cccccccc}
    0 & 0 & 0 & 0 & -2 & 2 l_2 & 0 & 0 \\
    0 & 0 & 4 l_2 & 4 & 0 & 0 & 2 a_2 & -2 a_2 l_2 \\
    0 & 0 & 0 & 0 & 0 & 0 & -2 l_2 & 2 \\
    4 & 4 l_2 & 0 & 0 & -2 a_2 l_2 & 2 a_2 & 0 & 0 \\
    \end{array}
    \right) \left(\begin{array}{c}
    a\\ b \\ \vdots \\ h
    \end{array}\right)=\left(\begin{array}{c}
    0\\ 0 \\ \vdots \\ 0
    \end{array}\right)$$
    
    \noindent Solving for the kernel of the above $4\times8$ coefficient matrix $A$, we get four solutions for $(a:b:\ldots:h)$ in $\PS^7(\R)$. Note that in computing for the kernel of $A$, we need to pay attention to entries of $A$ that involve the DH-parameters since they may assume the value of 0. Hence, the linear 3-space defined by
    \begin{align*}    
    -l_2 x_0+x_1&=0 \\
    a_2 \left(l_2^2-1\right) x_0+2l_2 y_0+2y_1&=0 \\
    a_2 \left(l_2^2-1\right) x_3+2y_2+2l_2 y_3&=0 \\
    x_2-l_2 x_3&=0 
    \end{align*}
    contains $V_1$ and lies in $S$ if and only if $l_2=\pm1$.
    
   	From the above linear space, we need to get the linear 3-space parametrized by joint $v_1$ that contain $V_L$, which we denote by $T(v_1)$. This is done by the following change of variables (operations are dual quaternion arithmetic and we view $(x_0:x_1:\dots : y_3)$ on the right as a dual quaternion)
    \begin{equation}\label{full_linear_space}
    (x_0:x_1:\dots : y_3) \rightarrow (\sigma_1(v_1)T_z(d_2))^{-1}(x_0:x_1:\dots:y_3)(R_z(v_3) T_x(a_3) R_x(l_3))^{-1}
    \end{equation}    
    where the inverse transformations are represented by the dual quaternion conjugates.
    
    Similar procedures can be applied to obtain the linear space parametrized by $d_3$ that also contain $V_L$, denoted $T(d_3)$. Here, the linear space that contain the kinematic image (leaving out the third joint)
	$$V_3:=\{R_{z}(v_1) T_{x}(a_1) R_{x}(l_1) R_{z}(v_2) \mid v_1,v_2\in\R \}$$ 
    is defined by
    \begin{equation}\begin{split}
    a_1 l_1 x_0 +2y_0 &= 0    \\
    -a_1 x_1 +2l_1 y_1 &=0  \\
    -a_1 x_2 +2l_1 y_2 &= 0 \\
    a_1 l_1 x_3 +2y_3 &= 0
    \end{split}\label{eq:left3R}\end{equation}
    when $a_1$ and $l_1$ are non-zero. If $a_1=l_1=0$ then we have the projective line   
    $$x_1=x_2=y_0=y_1=y_2=y_3=0.$$
The linear 3-space defined by equations~\eqref{eq:left3R} lies inside $S$ if and only if either $a_1$ or $l_1$ is 0. 
The linear space parametrized by $d_3$, $T(d_3)$, is then obtained by accounting for the fixed transformations and the parametrizing joint variable $d_3$, i.e. by applying the following change of variables
$$(x_0:x_1:\dots : y_3) \rightarrow (x_0:x_1:\dots:y_3)(T_z(d_2) T_x(a_2) R_x(l_2)\sigma_3(d_3))^{-1}.$$
The DH parameter values for which both $T(v_1)$ and $T(d_3)$ lie in $S$ are 
$$\{a_1=0 \textnormal{ or } l_1=0\} \textnormal{ and } l_2=\pm1$$
For convenience, we will assume that the given DH parameters do not satisfy these values so that $T(v_1)$ or $T(d_3)$ 
is not contained in $S$. When this assumption is not true, we still need to compute for $T(v_2)$ and the DH-parameter values for which $T(v_2)$ lies in $S$. To save space, we have not included this case in this paper. A detailed discussion can be found in \cite{manongsong}.
    
\vspace{5mm}
The procedures described in this section can be applied analogously to compute linear spaces that contain the 
kinematic image of the 3R joint type with joint variables $v_1,v_2$ and $v_3$. We only need to verify that the 
linear space containing
$$V_1:=\{R_{z}(v_2) T_{x}(a_2) R_{x}(l_2) R_{z}(v_3) \mid v_2,v_3\in\R\}$$ 
is the set of vanishing points of linear forms similar to
equation~\eqref{eq:left3R} but by replacing $a_1$ with $a_2$ and 
$l_1$ with $l_2$ (we assume $a_2$ and $l_2$ are non-zero). For RRR, the kinematic image $V_3$ is the same as for 
RRP, so equation~\eqref{eq:left3R} also describes the linear 3-space containing $V_3$. To obtain $T(v_1)$ or $T(v_3)$ we 
account for the fixed transformation and the parametrizing joint variables by substituting variables similar to equation~\eqref{full_linear_space}. This is also described in \cite{pfurner}. For simplicity we will assume that $a_6=d_6=l_6=0$ (otherwise the linear space is another 
easy change of variables taking these fixed transformations into account).

 \section{Computing hyperplanes for the right chain}

Recall that $\sigma_E$ is the pose of the end-effector. In order to obtain parametrized linear spaces that contain the kinematic image of the right chain
$$V_R:=\{\sigma_E\sigma_6^{-1}(v_6)\sigma_5^{-1}(v_5) \sigma_4^{-1}(v_4) \mid v_4,v_5,v_6 \in \R \}$$ 
we do the following steps.
\begin{enumerate}[1.]
	\item Consider the ``reverse joint type" of the right chain. For instance, in a 3RP2R manipulator the reverse joint type of the right chain is RRP. 
	\item Depending on the joint, obtain parametrized linear spaces $T(v_i)$, $i=1,3$, that contain its workspace.
	\item Apply the following parameter substitutions 
	\begin{equation}\begin{aligned}
		&v_1\rightarrow-v_6,\quad a_1\rightarrow-a_5,\quad l_1\rightarrow-l_5\\
		&v_2\rightarrow-v_5,\quad a_2\rightarrow-a_4,\quad l_2\rightarrow-l_4,\quad d_2\rightarrow-d_5\\
		&v_3\rightarrow-v_4, \quad a_3\rightarrow 0, \qquad\, l_3\rightarrow 0,\quad\quad d_3\rightarrow-d_4.\\
	\end{aligned}\label{eq:substitutions}\end{equation}
	
	\item Perform a change of variables
	$$(x_0:x_1:\dots : y_3) \rightarrow \sigma_E^* (x_0:x_1:\cdots:y_3)$$
\end{enumerate}

In a 2RP3R manipulator, the parametrized linear spaces $T(v_4)$ and $T(v_6)$ obtained from $T(v_3)$ and $T(v_1)$, respectively, each contain $V_R$. Moreover $T(v_4)$ lies in $S$ if and only if $a_5$ or $l_5$ is 0, while $T(v_6)$ lies in $S$ if and only if $a_4$ or $l_4$ is 0. For convenience, we assume that the given DH parameters do not satisfy $$\{a_4=0 \textnormal{ or } l_4=0\} \textnormal{ and } \{a_5=0 \textnormal{ or } l_5=0\}$$
so that neither $T(v_4)$ nor $T(v_6)$ is contained in $S$. It is also possible to compute $T(v_5)$, but we do not include this case in this work (for more details, see \cite{manongsong}).

Obtaining these parametrized linear spaces is an important step in the HuPf algorithm to solve the inverse kinematics
problem for 2RP3R manipulators. In this setting and with the information that we have now obtained, the IK problem reduces to a linear algebra problem. This will become clear in the next section.

\section{Solving the IK problem of 2RP3R manipulators}

The main idea for the inverse kinematics of 2RP3R manipulators is to obtain points in the intersection of left and right workspaces 
in a suitable ambient space (see proof of Proposition~\ref{prop:finite}). This entails  solving a system of nine constraint equations: 
eight linear forms and the Study equation. For the proposed IK procedure to work, both parametrized linear spaces must not lie in $S$. This means we need to check the given DH parameters as follows: if $l_2\neq\pm1$, use $T(v_1)$, otherwise, use $T(d_3)$; if $(a_4,l_4)\neq(0,0),$ use $T(v_6)$, otherwise, use $T(v_4)$. Let $T(u)$ and $T(w)$ be the chosen parametrized linear spaces for the left and 
right workspaces, respectively (so $u\in\left\{v_1,d_3\right\}$ and $w\in\left\{v_4,v_6\right\}$). In the end, we want to obtain 
finite complex solutions (thus also real solutions will be finite) to the problem.  We argue using 
complex solutions so that we can use 
basic results in classical algebraic geometry rather than rely on more sophisticated results in real algebraic geometry.

\begin{proposition}\label{prop:finite}
	Given the DH-parameters of a general 2RP3R manipulator. If its set of complex IK solutions is finite, then the eight 
  hyperplanes from $T(u)$ and $T(w)$ must be in \textit{general position} i.e.\ the $8$ hyperplanes are described by 
  linear forms over $\C[u,w]$ such that coefficients of these linear forms are linearly independent vectors in an $8$-
  dimensional vector space over $\C(u,w)$.
\end{proposition}

\begin{proof}
We regard $T(u)$ and $T(w)$ as subvarieties of $\PS^7(\C)\times \PS^1(\C)\times \PS^1(\C)$ (i.e.\ a Segre subvariety, 
see \cite[\S2\, pp.27]{harris}) and denote their intersection by $\mathcal W$. We now prove by contradiction and assume 
that we cannot find eight hyperplanes from $T(u)$ and $T(w)$ in general position. In particular, $\mathcal W$ will 
contain a parametrized family
$$\bigcup_{\alpha,\beta\in \mathbb P^1} X(\alpha,\beta) \subset \PS^7\times \PS^1\times \PS^1$$
where $X(\alpha,\beta)$ is a finite set of points in $\PS^7$ belonging to a section of $\mathcal W$ with the linear 
space $\PS^7\times \{(\alpha,\beta)\}$. Since we assume the set of IK solutions is finite, the fiber of the projection to $\PS^1\times \PS^1$ which is $X(\alpha,\beta)$ is also finite.  So the canonical projection of the
$\mathcal W$ to $\PS^7$ contains a surface (two-dimensional). 
 
This 2-dimensional projection has a non-empty intersection with the 6-dimensional Study quadric $S$ \cite[Chapter~I. Theorem~7.2]{harts}. In particular, this intersection contains a curve in $S$. 
However, this intersection yields an infinite subset of the set of solutions to the IK problem of the manipulator, 
a contradiction.
\end{proof}

The above proof makes use of an argument involving dimensions of varieties. A discussion on dimensions of varieties can be found in many introductory books in 
algebraic geometry (see for instance \cite[\S11]{harris}).  

\begin{remark}\label{rem:step2}
We actually require that the projection of $\mathcal W$ in the proof of Proposition~\ref{prop:finite} does not lie in the Study quadric. This is used in \ref{step2} where we chose a suitable seven out of eight linear forms. 
\end{remark}

The following discussion shows our procedure for solving the IK based on the HuPf algorithm (we assume throughout that there is at least one real IK solution to a given end-effector pose):

\begin{none}\label{step1}
Consider eight linear forms describing $T(u)$ and $T(w)$ as polynomials in eight variables (i.e.\ $x_0,x_1,\dots,y_3$) with coefficients in $\C[u,w]$. The coefficient matrix of this system of equations with entries in $\C[u,w]$ must be non-singular for the procedure to work.
If the coefficient matrix is singular then, by the proof of Proposition~\ref{prop:finite}, there will be infinite solutions to the IK problem. Hence, we only proceed when it is nonsingular.

\vspace{3mm}
We choose seven out of the eight linear forms and regard their vanishing points as hyperplanes in $\PS^7(\C(u,w))$. 
Using linear algebra, we can easily solve for the intersection of these hyperplanes in $\PS^7(\C(u,w))$. We 
denote this point $P(u,w)$ with coordinates 
$$P(u,w):=(x_0(u,w):\cdots:y_3(u,w))$$
We may clear denominators and assume that the coordinates are in $\C[u,w]$ and not all are 
zero. The point $P$ will eventually give us all solutions to the pose of $F_4$ for the given EE 
transformation. 
\end{none}

\begin{none}\label{step2}
We now substitute the coordinates of $P$ into the quadratic form defining $S$. We obtain a bivariate polynomial 
$$f(u,w) := \sum_{i=0}^3 x_i(u,w)y_i(u,w)$$
in $\R[u,w]$. If this polynomial is identical to $0$ then we will choose another seven linear form in \ref{step1} and recompute $P(u,w)$ and again substitute in $S$. We assume we can find 
seven linear forms for which the polynomial is not identical to $0$ 
(see Remark~\ref{rem:step2}), otherwise one would need a different algorithm to solve the inverse 
kinematics problem (this possibility was handled by \cite{pfurner} for 6R manipulators). In 
this case, one may still have finite solutions for the IK problem. Hence, we may 
assume that $f(u,w)$ is non-zero.

Substituting the coordinates of $P$ into the unaccounted linear form (recall we have $8$ linear forms from $T(u)$ and $
T(v)$ and we only used $7$ of them to find $P$) should yield a non-constant 
$g(u,w)\in \R[u,w]$. The reason why $g$ is non-constant is because we have an IK solution and a constant 
$g$ would imply that $g\equiv 0$. However, $g\equiv 0$ implies that the eight hyperplanes from $T(u)$ and $T(w)$ 
are not in general position and this contradicts the conclusion of Proposition~\ref{prop:finite}. But this case was 
already eliminated in \ref{step1}.
\end{none}

\begin{none}\label{step3}
Take the resultant (see \cite{cox}) of $f(u,w),g(u,w)\in \R[u,w]$ from \ref{step2} by viewing them as multivariate polynomials over 
the ring $\R[u]$. The resultant will thus itself be a polynomial $r(u)\in \R[u]$. Since we are only considering 
solutions to the IK problem where the rotations are not an odd multiple of $\pi$, the solution to the joint variables 
$u$ and $w$ will be in the intersection of the plane curves defined by $f$ and $g$. The resultant $r$ cannot be 
identical to $0$ because we have eliminated this possibility in \ref{step1} and \ref{step2}. Finally, the resultant 
is not a non-zero constant because we know there is a solution to the IK problem. 

The finite number of roots of $r$ (values for $u$ only) will yield the possible values for the joint variable $u$. In applications, only the real roots of 
$r$ are of interest. Substituting these real roots to $f$ and $g$ gives us pairs of polynomials in $\R[w]$ and their 
common real root will give us possible values for $w$ for a given real value for $u$. Thus we obtain all the real 
intersections of the plane curves defined by $f$ and $g$ respectively. 

The corresponding $P(u,w)$ for joint variables $u$ and $w$ need to be in $\SE$. Namely, we discard all 
values from the possible pairs $(u,w)$ such that 
$$x_0(u,w)=x_1(u,w)=x_2(u,w)=x_3(u,w)=0$$
\end{none}

\begin{none}
The points $P(u,w)$ computed so far are points in the intersection of the left and right workspaces for which the left 
and right chains may coincide at frame $F_4$. The pairs $(u,w)$ comprise two of the 6 joint variables needed for the 
IK problem. For each $(u,w)$ solved in \ref{step3}, the other four joint values on the left and the right chain can be computed as follows: 
\begin{enumerate}[1.]
\item We first solve the unknown joint value of the left chain that is not $v_2$. Say, 
if the unknown joint is $d_3$ (i.e. $u=v_1$), then choose one linear form from $T(d_3)$ 
and solve for $d_3$ by substituting $P(u,w)$ (if unknown joint variable is $v_1$ then we choose linear form 
from $T(v_1)$).
\item To solve for $v_2$, we need to first determine $T(v_2)$ (i.e. linear space parametrized by $v_2$ containing the left workspace $V_L$). We perform the following steps:
\begin{enumerate}[(a)]
  \item With the given DH-parameters, compute the Study parameters of $V_L$ i.e.\, compute
$$\sigma_1(v_1)\sigma_2(v_2)\sigma_3(d_3)$$
considered as an element in $\PS^7(\C(v_1,v_2,d_3))$ where each coordinates are polynomials in $\C[v_1,v_2,d_3]$. Suppose the Study parameters are 
$$(s_0(v_1,v_2,d_3):s_1(v_1,v_2,d_3):\cdots: s_7(v_1,v_2,d_3)).$$
\item Substitute the Study parameters ($s_0,s_1,\dots, s_7$) into 
\begin{equation}\label{eq:Tv2}
(a+iv_2)x_0+(b+jv_2)x_1+(c+kv_2)x_2+\cdots +(g+ov_2)y_2+(h+pv_2)y_3
\end{equation}
and rewrite as a polynomial in $v_1,v_2$ and $d_3$ with coefficients $z_i$, $i=1,2,\ldots,12$.
  \item Create a $12\times 16$ matrix $B$ where each row $i$ consists of the coefficients of $a,b,\ldots,p$ in $z_i$.
  \item Determine an element of $\ker(B)$ such that when it is substituted 
  (i.e. $a,b,\dots, p$) in equation \eqref{eq:Tv2} one obtains a linear form parametrized by $v_2$ (i.e. one of the $i,j,\dots, p$ is non-zero). Such an element of $\ker(B)$ exists otherwise there are infinite solutions to $v_2$ and we assumed only finite solutions to the IK problem.
  \item Substitute $P(u,w)$ to the obtained linear form and solve for $v_2$.
\end{enumerate}
\item The above steps can be done analogously for finding the unknown joint values of the right chain wherein in Item 1 we solve the joint value that is not $v_5$ and in Item 2 we compute the Study parameters of $V_R$ and substitute them into an equation similar to \eqref{eq:Tv2} but linear in $v_5$. Thus one obtains a linear form parametrized by $v_5$.
\end{enumerate}
	
\end{none}

\begin{none}
	Finally, substitute all possible real IK solutions, e.g. $(v_1,v_2,d_3,v_4,v_5,v_6)$ for the 2RP3R, into $\sigma_1\sigma_2\cdots\sigma_6$. The real IK solutions are the scalar multiples of $\sigma_E.$
\end{none}

Analogous procedures in this section can be applied for solving the IK of 2R2P2R, 3RP2R and 6R manipulators. For instance, for the 2R2P2R manipulator one solves joint values $(u,v)$ on the left and right chain, respectively. In particular, on the right chain, one chooses a linear form from $T(d_4)$ if $v=v_6$ and solves for $d_4$. Thus, the IK solutions are values for $(v_1,v_2,d_3,d_4,v_5,v_6)$.

\section{Examples}

Consider a 2RP3R manipulator with DH-parameters given in Table~\ref{tab:2rp3rDH}.
The parametrized linear spaces $T(v_1)$ and $T(v_6)$ are not in the Study quadric.   

\begin{table}[htbp!] \centering
	\begin{tabular}{|c|c|c|c|c|}
		\hline
		$i$ & $\theta_i$(deg) & $d_i$ & $a_i$ & $\alpha_i$(deg)\\
		\hline
		1   & * & 0  & 0.1 & 90 \\ 
		2   & * & 0  & $-$0.425 & 0 \\
		3   & $0$ & *    & $-$0.39225 & 0  \\
		4   & * & 0.10915  & 0.01 & 90 \\
		5   & * & 0.09465  & 0 & $-$90 \\
		6   & * & 0 & 0 & 0 \\
		\hline
	\end{tabular}	
	\caption{DH-parameters}
  \label{tab:2rp3rDH}
\end{table}

\vspace{3mm}\noindent
Take a desired pose of the end-effector given by
$$(190.335, 213.413, 9.36544, 164.774, -35.3968, -32.883, 74.4773,79.2444).$$ 
Since $l_2\neq\pm1$, choose $T(v_1)$ for the left chain. Since $a_5=0$ we have $T(v_4)\subset S$, hence we choose 
$T(v_6)$ for the right chain. We obtain 4 real solutions to the IK problem as shown in Table~\ref{tab:sols1} and Figure~\ref{fig1}. These solutions were obtained using Mathematica (codes are available in~\cite{link}).

\begin{table}[!ht] \centering 
	\begin{tabular}{|c|c|c|c|c|}
		\hline
		& Solution 1 & Solution 2 & Solution 3 & Solution 4\\ 
		\hline
		$\theta_1$ &$-$16.1819 & 40.9555 & 60 & 79.2813 \\
		$\theta_2$ &$-$70.8614 & $-$58.2515 & $-$70 & $-$71.455 \\
		$d_3$ &0.0810177 & $-$0.123834 & $-$0.2 & $-$0.266949 \\
		$\theta_4$ &81.6927 & 133.782 & 40 & 55.7351 \\
		$\theta_5$ &$-$60.0253 & $-$9.67835 & 19 & 36.9289 \\
		$\theta_6$ &32.9097 & $-$36.9601 & 67 & 51.0502 \\
		\hline
	\end{tabular}\normalsize  
	\caption{Real inverse kinematics solutions to the given 2RP3R manipulator} \label{tab:sols1}
\end{table}

\begin{figure}[htbp!]\centering
\subfloat[Solution 1]{\fbox{\includegraphics[width=45mm]{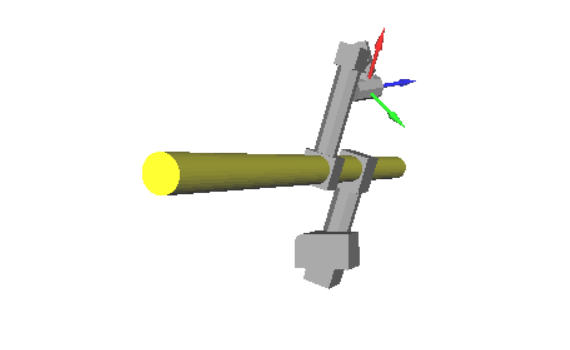}}}\hspace{1mm}
\subfloat[Solution 2]{\fbox{\includegraphics[width=45mm]{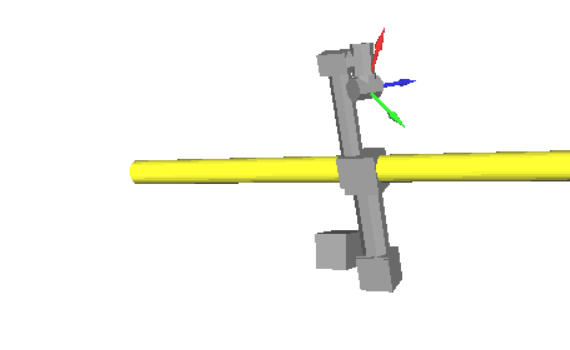}}}\\[-1.5mm]
\subfloat[Solution 3]{\fbox{\includegraphics[width=45mm]{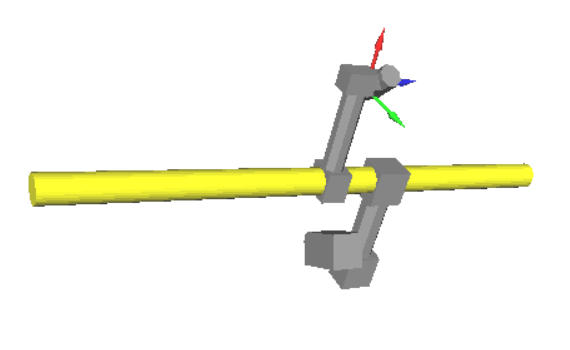}}}\hspace{1mm}
\subfloat[Solution 4]{\fbox{\includegraphics[width=45mm]{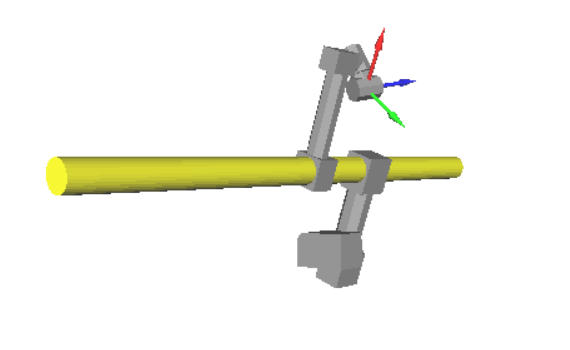}}}
\caption{Illustration of IK solutions to the 2RP3R manipulator}
\label{fig1}
\end{figure}

\noindent Consider a 2R2P2R manipulator with given DH parameters in Table~\ref{tab:2r2p2r} and EE pose
$$(-5.37543, 64.9811, -75.9243, 69.384, -59.0113, 6.15132, -22.5377, -34.995).$$

\begin{table}[ht]\centering
	\begin{tabular}{|c|c|c|c|c|}
		\hline
		$i$ & $\theta_i$(deg) & $d_i$ & $a_i$ & $\alpha_i$(deg) \\
		\hline
		1   & * & 0  & 0.2 & 23  \\
		2   & * & 0.3  & 0.2 & 23  \\
		3   & $-$45 & *    & 0.3 & 45  \\
		4   & 71  & *  & 0.4 & 35  \\
		5   & * & 0.3  & 0 & 20  \\
		6   & * & 0 & 0 & 0 \\
		\hline
	\end{tabular}  
	\caption{DH-parameters }
  \label{tab:2r2p2r}
\end{table}
 \noindent The IK solutions can be computed from $T(v_1)$ and $T(v_6)$ since both do not lie in $S$. Four real solutions are obtained and are shown in Table~\ref{tab:solutionsO} and Figure~\ref{fig2}. 

\begin{table}[ht!] \centering 
	\begin{tabular}{|c|c|c|c|c|}
		\hline
		& Solution 1 & Solution 2 & Solution 3 & Solution 4 \\
		\hline
		$\theta_1$& $-$4.2843 & 10. & 104.328 & 128.106 \\
		$\theta_2$&70.5556 & 20. & $-$64.1116 & $-$72.0589 \\
		$d_3$  &$-$0.378641 & 0.1 & 0.0119017 & $-$0.33362 \\
		$d_4$  &0.639115 & $-$0.1 & 0.579609 & 1.03798 \\
		$\theta_5$&$-$110.876 & 31. & 10.0365 & $-$35.978 \\
		$\theta_6$&$-$167.772 & 55. & 107.717 & 157.911 \\
		\hline  
		
	\end{tabular} \normalsize 
	\caption{Real inverse kinematics solutions to the given 2R2P2R manipulator  }\label{tab:solutionsO}
\end{table}

\begin{figure}[!htbp]\centering
\subfloat[Solution 1]{\fbox{\includegraphics[width=5cm]{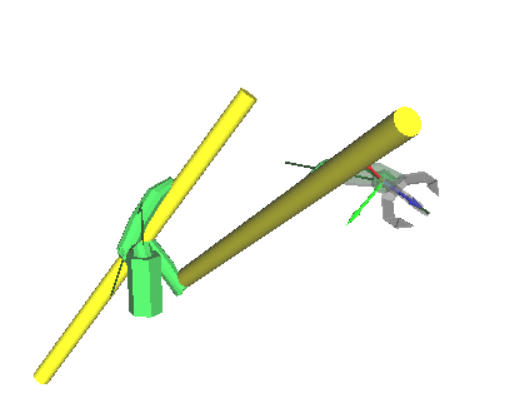}}}\hspace{1mm}
\subfloat[Solution 2]{\fbox{\includegraphics[width=5cm]{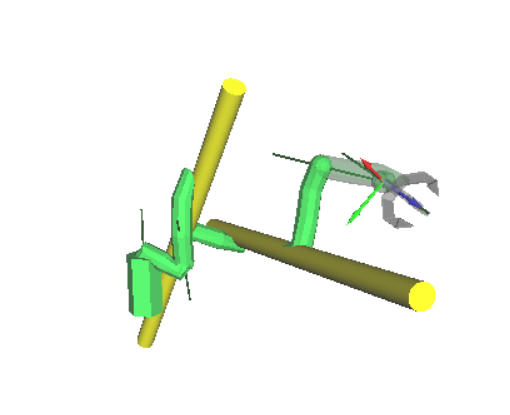}}}\\[-1.5mm]
\subfloat[Solution 3]{\fbox{\includegraphics[width=5cm]{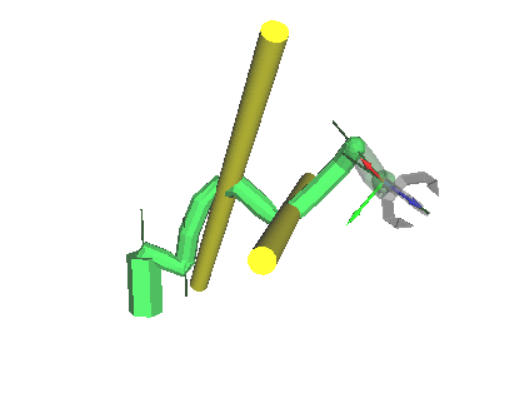}}}\hspace{1mm}
\subfloat[Solution 4]{\fbox{\includegraphics[width=5cm]{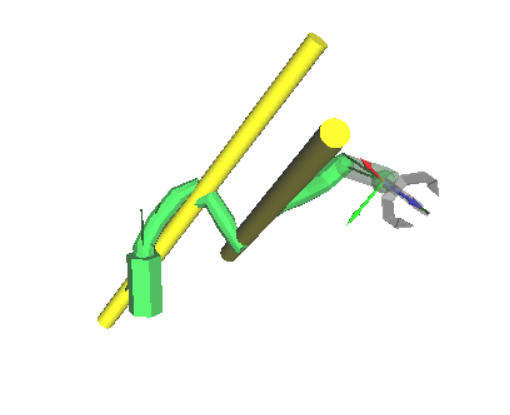}}}
\caption{Illustration of IK solutions to the 2R2P2R manipulator}
\label{fig2}
\end{figure}

\section{Conclusion}
A 2RP3R manipulator can be described as a composition of two rotations followed by a 
translation and three rotations in space. Computing the constraint varieties of such manipulators can be greatly 
simplified by considering the two 3-subchains namely 2RP and 3R. The parametrized linear spaces (linear sections of a 
Segre variety) describing the workspaces of these 3-chains are the key to the solution of the IK problem. 
Computing the intersection of a quasi-projective variety (identified with $\SE$) with the projection of a linear section of Segre varieties using elimination theory and linear algebra allows us to solve the 
IK problem algebraically.

This algorithm can be applied to all general 6R/P manipulators with the same 3-subchain types (one takes the reverse 
joint type for the right 3-chain), namely: 2RP, 3R. Our algorithm differs from the original HuPf algorithm because it accounts 
for prismatic joints, and we provide a systematic and efficient way (via dual quaternions) of finding and choosing the proper parametrized linear spaces. Note that in this paper we only considered four different 6R/P manipulators. The methods discussed here can be modified so that they are applicable to other 6R/P manipulators, and they will be dealt with in a forthcoming paper. 

\section*{Acknowledgements}

J.~Capco was supported and funded by the Austrian Science Fund (FWF): Project P28349-N32 and W1214-N14 Project DK9.  
M.J.C.~Loquias and S.M.M.~Manongsong acknowledges the Office of the Chancellor of the University of the Philippines 
Diliman, through the Office of the Vice Chancellor for Research and Development, for funding support through the Outright Research Grant.

\end{document}